\documentclass[10pt,twocolumn,letterpaper]{article}

\usepackage[pagenumbers]{cvpr} 

\usepackage{graphicx}
\usepackage{amsmath}
\usepackage{amssymb}
\usepackage{booktabs}

\usepackage{mathtools}

\usepackage[toc]{appendix}

\usepackage[pagebackref,breaklinks,colorlinks,hypertexnames=false]{hyperref}

\usepackage{multicol}
\usepackage{multirow}
\usepackage[dvipsnames,svgnames,x11names]{xcolor}
\usepackage{mathtools}
\usepackage{xspace}
\usepackage{makecell}

\usepackage[capitalize]{cleveref}
\crefname{section}{Sec.}{Secs.}
\Crefname{section}{Section}{Sections}
\Crefname{table}{Table}{Tables}
\crefname{table}{Tab.}{Tabs.}

\newcommand{\fig}[1]{\cref{fig:#1}}
\newcommand{\topic}[1]{\vspace{1mm}\noindent\textbf{#1}}
\newcommand{\FLIP}{\protect\reflectbox{F}LIP\xspace}

\newcommand{\modelname}{StereoLayers\xspace}
\newcommand{\baselinestereomag}{StereoMag\xspace}

\usepackage{amsthm}

\newtheorem{theorem}{Theorem}[section]

\newtheorem{lemma}[theorem]{Lemma}
\newtheorem*{lemma*}{Lemma}

\newif\ifArxivVersion
\ArxivVersiontrue

\begin{document}

\title{Stereo Magnification with Multi-Layer Images}

\author{
    T. Khakhulin$^{1,2}$
    \; D. Korzhenkov$^{1}$
    \; P. Solovev$^{1}$
    \; G. Sterkin$^{1}$
    \; A.-T. Ardelean$^{1,2}$
    \; V. Lempitsky$^{2}$\thanks{Most of the work was done while Victor Lempitsky was at Samsung AI Center}
    \\[7pt] $^1$Samsung AI Center -- Moscow
    \\ $^2$Skolkovo Institute of Science and Technology, Moscow
    \\[7pt]  \url{https://samsunglabs.github.io/StereoLayers/}
}

\maketitle

\begin{abstract}
Representing scenes with multiple semitransparent colored layers has been a popular and successful choice for real-time novel view synthesis. 
Existing approaches infer colors and transparency values over regularly spaced layers of planar or spherical shape. In this work, we introduce a new view synthesis approach based on multiple semitransparent layers with scene-adapted geometry. Our approach infers such representations from stereo pairs in two stages. The first stage produces the geometry of a small number of data-adaptive layers from a given pair of views. The second stage infers the color and transparency values for these layers, producing the final representation for novel view synthesis. Importantly, both stages are connected through a differentiable renderer and are trained end-to-end. In the experiments, we demonstrate the advantage of the proposed approach over the use of regularly spaced layers without adaptation to scene geometry. Despite being orders of magnitude faster during rendering, our approach also outperforms the recently proposed IBRNet system based on implicit geometry representation.
\end{abstract}

\section{Introduction}

\begin{figure*}[t]
    \centering
    \includegraphics[width=\textwidth]{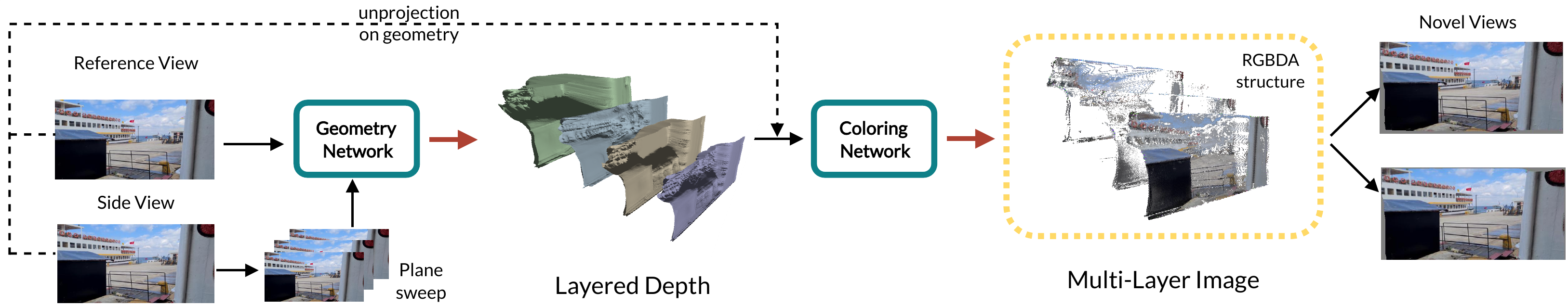}
    \caption{The proposed \modelname pipeline estimates scene-adjusted multi-layer geometry from the plane sweep volume using a pretrained geometry network, and after that estimates the color and transparency values using a pretrained coloring network. The layered geometry represents the scene as an ordered set of mesh layers. The geometry and the coloring networks are trained together end-to-end. 
    }
    \label{fig:scheme}
    \vspace{-12pt}
\end{figure*}

Recent years have seen rapid progress in image-based rendering and novel view synthesis, with a multitude of various methods based on neural rendering approaches~\cite{neuralsurvey}. Among this diversity, the approaches that are based on semitransparent multi-layer representations \cite{szeliski99,stereomag,pushing_bound,llff,singleview_mpi} stand out due to their combination of fast rendering time, compatibility with traditional graphics engines, and good quality of re-rendering in the vicinity of the input frames.

Existing approaches \cite{szeliski99,stereomag,pushing_bound,llff,singleview_mpi,immersive_lf_video,multidepth_panorama} build multi-layer representations over grids of regularly spaced surfaces such as planes or spheres with uniformly changing inverse depth. 
As the number of layers is necessarily limited by resource constraints and the risk of overfitting, this number is usually taken to be relatively small (\eg 32). 
The resulting semi-transparent representation may therefore only coarsely approximate the true geometry of the scene, which limits the generalization to novel views and introduces artefacts. The most recent works~\cite{multidepth_panorama,immersive_lf_video} use excessive number of spheres (up to 128) and then merge the resulting geometry using a non-learned post-processing merging step.
While the merge step creates scene-adapted and compact geometric representation, it is not incorporated into the learning process of the main matching network, and degrades the quality of novel view synthesis~\cite{immersive_lf_video}. 

The coarseness of layered geometry used by multi-layer approaches is in contrast to more traditional image-based rendering methods that start by estimating the \textit{non-discretized} scene geometry in the form of mesh~\cite{deferred_rendering,free_synthesis}, view-dependent meshes~\cite{deep_blending}, a single-layer depth map~\cite{synsin,ken_burns,3d_layered_inpainting}. Geometry estimates may come from multiview dense stereo matching or from monocular depth. All these approaches obtain a finer approximation to scene geometry, although most of them have to use a relatively slow neural rendering step to compensate for the errors in the geometry estimation.

Our approach called \emph{\modelname} (\fig{scheme}) combines scene geometry adaptation with multi-layer representation. This model is designed for a case known as \emph{stereo magnification} problem: it reconstructs the scene from as few as two input images. The proposed method starts by building a geometric proxy that is customized to a particular scene. The proxy is formed by a small number of mesh layers with \textit{continuous} depth coordinate values. In the second stage, similarly to other multi-layer approaches, we estimate the transparency and color textures for each layer, resulting in the final representation of the scene. When processing a new scene, both stages take the same pair of images of that scene as input. Two deep neural networks trained on a dataset of similar scenes are used to implement these two stages. Crucially, we train both neural networks together in an end-to-end fashion using the differentiable rendering framework~\cite{Laine2020diffrast}.

We compare our approach to the previously proposed methods that use regularly spaced layers on the popular RealEstate10k~\cite{stereomag} and LLFF~\cite{llff} datasets. In addition, we propose a more challenging new dataset for novel view synthesis benchmarking. In both cases, we observe that scene-adaptive geometry in our approach results in better novel view synthesis quality than the use of non-adaptive geometry. To put our work in a broader context, we also compare our system's performance with the IBRNet system~\cite{ibrnet}, and observe the advantage of our approach, in addition to the considerably faster rendering time. In general, our approach produces very compact scene representations that are amenable for real-time rendering even on low-end devices.

To sum up, our contributions are as follows.
First, we propose a new method for the geometric reconstruction of a scene from pairs of stereo. The method represents scenes using a small number of semitransparent layers with scene-adapted geometry. Unlike other related methods, ours uses two jointly (end-to-end) trained deep networks, the first of which estimates the geometry of the layers, while the second estimates the transparency and color textures of the layers. Finally, we evaluate our approach on previously proposed datasets and introduce a new challenging dataset for training and evaluation of novel view synthesis methods.

\section{Related works}

\begin{figure*}[ht]
    \centering
    \includegraphics[width=\textwidth]{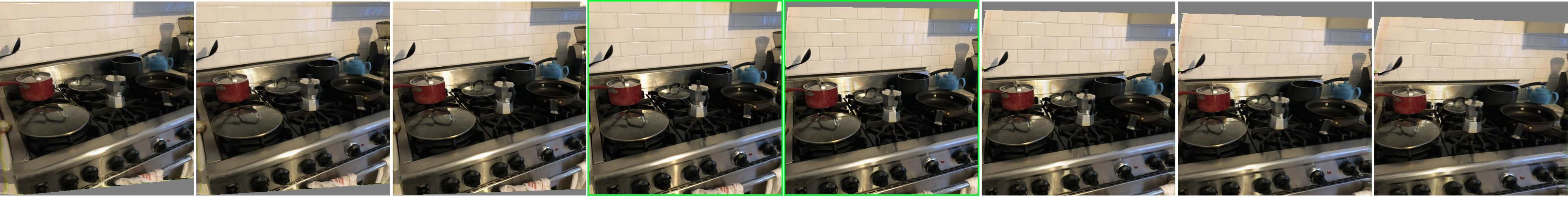}\\
    \includegraphics[width=\textwidth]{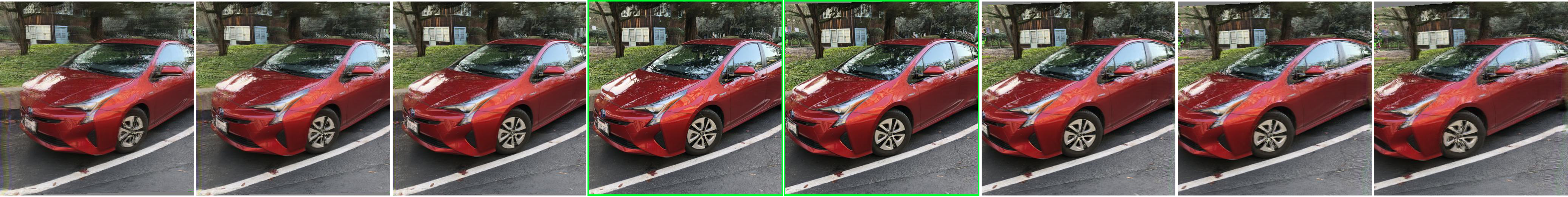}
    \caption{View extrapolations obtained by our method. The two input images are shown in the middle. The proposed method (StereoLayers) generate plausible renderings even when the baseline is magnified by a factor of 5x (as in this case). 
    }
    \label{fig:magnification}
    \vspace{-10pt}
\end{figure*}

\topic{Representations for novel view synthesis.} Over the years, different kinds of representations have been proposed for novel view synthesis. Almost without exception, when such representations are acquired from multiple images, those are registered using structure-and-motion algorithm or come from a pre-calibrated stereo-rig. Alternatively, some recent works investigate the creation of such representations from a single image~\cite{lsi,synsin}. 
The proposed representations fall into several classes, including volumetric representations that rely on volumetric rendering~\cite{lumigraph,deepvoxels,neuralvolumes,nerf}, mesh-based representations~\cite{debevec96,zitnick04,deferred_rendering,insideout,deep_blending,free_synthesis} and point-based representations~\cite{pointbasedsurvey,npbg}. Most representations of these types require extensive computations to render a novel view, such as running a raw image through a deep convolutional rendering network~\cite{neuralsurvey} or numerous evaluations of a scene network that has a perceptron architecture~\cite{nerf,nsvf}.

An important class of representations is based on depth maps. Such depth maps can be naturally obtained using stereo matching~\cite{chen_williams} or from monocular depth estimation~\cite{synsin,3d_layered_inpainting}. In this class, the 3D layered inpainting approach~\cite{3d_layered_inpainting} is most related to our work, since after starting from a monocular depth map, it performs its segmentation into multiple layers and then applies the inpainting procedure to each layer to extend its support behind the more frontal layers. Our work has several important differences, as it uses two (rather than one) images as input and predicts the transparency of the layers. Most importantly, the estimation of the multi-layered geometry and the estimation of their colors and transparency are both implemented using deep architectures, which are trained in an end-to-end fashion.

\topic{Multi-layer semitransparent representations.} In 1999, \cite{szeliski99} proposed representing scenes with multiple fronto-parallel semitransparent layers and acquiring such representations through stereo-matching of a pair of input views. Twenty years later, several approaches~\cite{pushing_bound, llff, deep_view} starting from \cite{stereomag} exploited advances in deep learning to build deep networks that directly map \textit{plane sweep volumes} (\ie tensors obtained by the ``unprojection'' operation) to final representations of the same kind. The rendering of semitransparent layers is well supported by modern graphics engines, thus the resulting representation is in general more suitable to interactive applications than most other representations that lead to the similar level of realism.

The multi-layer representations have been extended to wider fields of view in \cite{immersive_lf_video,matryoshka,multidepth_panorama} by replacing planes with spheres. Two approaches~\cite{immersive_lf_video,multidepth_panorama} suggested to ``coalesce'' (merge) the groups of nearby layers into layers with scene-adapted geometry. In both cases, the grouping of layers is predefined and the merge process is non-learnable and uses simple accumulation heuristics. Consequently, \cite{immersive_lf_video} reported the loss of rendering quality as a result of such merge, which is still justified in their case by increased rendering and storage efficiency. 

Our research is highly related to previous works on multi-layer semitransparent representations. Unlike most works in this group, our pipeline starts with scene-adapted (non-planar, non-spherical) layer estimation and only then estimates the colors and transparencies of the layers. While \cite{immersive_lf_video,multidepth_panorama} also end up with scene-adapted semi-transparent layers as a representation, our approach performs the reconstruction in the opposite order (the geometry is estimated first). More importantly, unlike~\cite{immersive_lf_video,multidepth_panorama} we estimate the geometry of layers using a neural network, which is trained jointly with the color and transparency estimation network. In the experiments, we show that such an approach results in better view synthesis. 

\topic{Single-layer new view synthesis with differentiable rendering.} SynSin~\cite{synsin} and, more recently, Worldsheet~\cite{worldsheet} systems predict single-layered geometry from a single image and use differentiable rendering to learn the neural network in a way similar to our method. Our approach considers the case of two input images and focuses on multi-layer geometry. While a variant of Worldsheet considers two-layer extension, it is based on a different architecture and a different layer aggregation strategy and, most importantly, does not outperform a single-layer representation in their experiments.

\section{Multi-layer representation from stereo}
\label{sec:method}
\begin{figure*}[ht]
\centering
    \includegraphics[width=\linewidth]{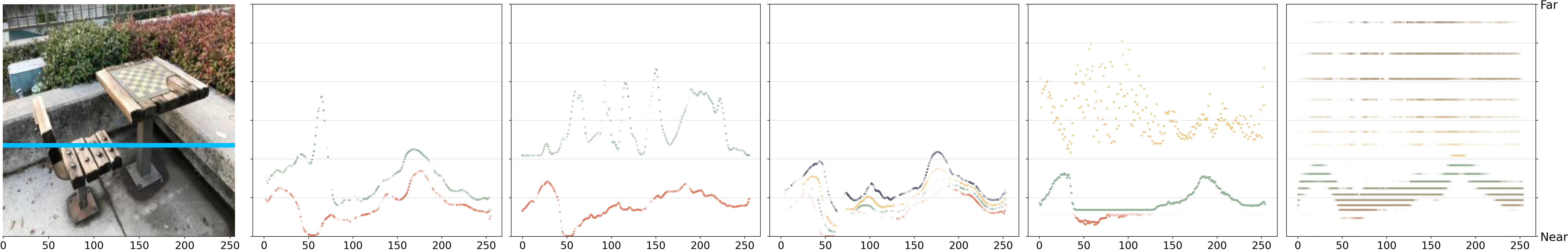}\\
    \includegraphics[width=\linewidth]{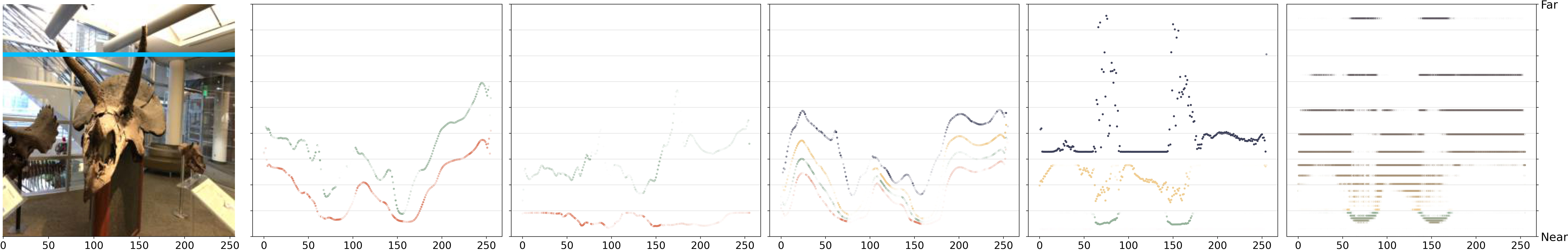}
\\ \makeatletter\@for\items:={Reference view,\modelname-2 (BI+RSBg),\modelname-2 (GC+RSBg),\modelname-4 (BI+RSBg),  \hspace{4pt}StereoMag-P4,  \hspace{10pt}StereoMag-32}\do{\minipage{0.16\textwidth}\centering\small \items \endminipage\hfill}\makeatother
    \vspace{-5pt}
    \caption{ 
        For the two stereo pairs (only reference views are shown), we visualize horizontal slices along the blue line. Mesh vertices are shown as dots with the predicted opacity. Colors encode the layer number. The horizontal axis corresponds to the pixel coordinate, while the vertical axis stands for the vertex depth \wrt the reference camera (only the most illustrative depth range is shown). StereoLayers method variants generate scene-adaptive geometry in a more efficient way than StereoMag~\cite{stereomag} resulting in more frugal geometry representation, while also obtaining better rendering quality.
    }
    \label{fig:layers_slices}
    \vspace{-10pt}
\end{figure*}

We consider the task of stereo magnification \ie generate a novel view $\hat{I}_n$ of the scene, based on two input views (images): a reference view $I_r$ and a side view $I_s$. We assume that the relative camera poses $\pi_s$ and $\pi_n$ of the side and novel views to the reference view and the camera intrinsics $K_r$, $K_s$, and $K_n$ are given.
To solve this task, our approach builds the scene representation that depends only on side and reference views. Afterward, such a representation can be rendered on any novel camera with standard graphic engines (without reestimating the scene representation). We now describe our approach in detail. We first explain the rendering procedure of a trained model and then discuss the training process.

\subsection{Geometry estimation}
Given a trained model and a new stereo pair, the multi-layer representation is inferred in two stages. First, the structure of the scene, such as the geometry of the mesh layers, is predicted. Then, in the second stage, the layers' opacity (alpha) and RGB color (textures) are inferred.
Note that we treat the pair of input views asymmetrically, as we build the scene representation in the frustum of the reference camera.

We start by computing the plane sweep volume (PSV)~\cite{Collins96} by placing $P$ fronto-parallel planes in the reference camera frustum and unprojecting the side view onto these planes. The planes are spaced uniformly in the inverse depth space at depths $\{ d_1, \dots, d_P \}$. We sample the planes at $H\times{}W$ resolution and concatenate the reference view as an additional set of three channels, resulting in $H\times{}W\times{}(3P+3)$-sized tensor, which is similar to the one used in other multi-layer approaches, \eg \cite{stereomag}.

The input tensor is then processed by the \emph{geometry network} $F_g$. Although we consider several variants of the architectures discussed in the following, all these architectures predict $L$ depth maps of size $h \times w$, which correspond to the depths along with the $h\times{}w$ pencil of rays uniformly spaced in the image coordinate space of the reference view.
In our experiments, we set the resolution of the layers equal to the size of the reference view, $w = W$, although sampling at different resolutions is also possible.
The backbone of $F_g$ is similar to the depth prediction module of SynSin~\cite{synsin}, \ie is a UNet-like 2D-convolutional net with spectral normalization.
The only difference is that we increased the number of input and output featuremaps to address the multi-layer nature of our model. More detailed description of the backbone is provided in Supplementary (\ifArxivVersion\cref{sec:architectures}\else Sec.~S1\fi). We consider the following three schemes to encode the layers.

\topic{Group compositing (GC) scheme.}  
In this scheme, $F_g$ returns the tensor of shape $h \times w \times P$ with values in the range between 0 and 1. The $P$ channels and the corresponding $P$ planes of PSV are divided into $L$ groups of equal size. Then $L$ deformable layers are obtained as follows:
within each group $j$, $1 \leq j \leq L$, the depth value $\hat{d}_j$ is computed by over-composing~\cite{compose_over} the planes' depths $d_1 < \ldots < d_P$ with `opacities' $\{\beta_k\}_{k=1}^P$ predicted by $F_g$ network. 
\begin{align}
\scalebox{0.95}{$ 
\hat{d}_j =\sum\limits_{k=I_j^-}^{I_j^+} d_k \beta_{k} \prod\limits_{i=I_j^-}^{k-1}\left(1-\beta_{i}\right), \quad j = 1,\dots,L,
$}
\label{eq:depth_composeover}
\end{align}
where $I_j^-$ and $I_j^+$ are the indices of the bounding planes for each group: $I_j^- = 1 + \left(j - 1\right)P/L,$ and $I_j^+ = jP/L. $
For simplicity of notation, layers and planes in \cref{eq:depth_composeover} are enumerated in front-to-back order.
The `opacities' $\beta_{I_j^+}$ (corresponding to farmost planes of each group) are manually set to 1 for $1 \leq j \leq L$.
As a result, the depth of the $j$-th layer is bounded by design: $d_{I_j^-} \leq \hat{d}_j \leq  d_{I_j^+}$. The compositing \cref{eq:depth_composeover} is evaluated independently for each of $h \times w \times L$ positions.

The group compositing scheme is inspired by the merge procedure from \cite{immersive_lf_video}, but moves this procedure inside the learnable scene representation and before texture and transparency estimation. The main benefit of the GC procedure is the guarantee that the $L$ layers do not intersect and have explicit ordering.

\topic{Soft-aggregation (SA) scheme.}
The main drawback of the GC depth aggregation is non-adaptive partition of the depth interval into $L$ layers. Such a non-adaptive partition tends to waste representation capacity for parts that do not contain scene surfaces and to underfit parts where multiple layers are beneficial for scene representation.
To overcome this, we make $F_g$ to predict the tensor of size $h \times w \times L \cdot P$, that is further reshaped to $h \times w \times L \times P$.
After that, softmax is applied along the last axis, and the values obtained are used as weights for the planes' depths $\{d_k\}_{k=1}^P$ (where the $P$ depths span the whole depth range).
These depths are averaged with the predicted weights, and the resulting tensor with the shape $h \times w \times L$ is obtained, which contains the depths of the layers.
It is worth noting that, unlike the GC approach, this scheme neither provides any ordering of layers, nor guarantees the absence of intersections.
Therefore, a special loss promoting non-intersection should be applied during training.
 
\topic{Bounds interpolation (BI) scheme.} 
We also consider a simplified version of SA scheme that predicts only weights to blend the minimum depth value $d_1$ and the maximum depth value $d_P$ (effectively predicting depths by direct regression).
In this scheme,  $F_g$ network returns the tensor of shape $h \times w \times L$ with values in the range between 0 and 1. 
The depth $\hat{d}_j$ of $j$-th layer is computed as
$\hat{d}_j = \beta_{j} d_1 + \left(1-\beta_{j}\right)d_P$,  where $\beta_j$ is the output of the geometry network. 
In our experiments, this scheme achieves the best results; thus, we select the BI method as our default one.

\topic{Meshing.} Irrespective of the layer depth prediction scheme, we treat each predicted layer as a mesh. We use the simplest mesh connectivity pattern, connecting each vertex with the six nearby nodes with edges so that each quad defined by four adjacent vertices is meshed with two triangles.
Hereinafter, the whole set of resulting $L$ meshes is referred to as the  \textit{layered mesh}.
The examples of estimated geometry are showcased in \cref{fig:layers_slices}.
Now we explain the explored approaches to depth prediction.

\begin{table*}[ht]
\centering
\begin{minipage}{0.95\linewidth}
\resizebox{\textwidth}{!}{
\begin{tabular}{lccccccccccccccccc}
\toprule
& \multicolumn{4}{c}{\bf SWORD }
& \multicolumn{5}{c}{\bf RealEstate10K}
& \multicolumn{5}{c}{\bf LLFF} \\
\cmidrule{2-5}  \cmidrule{7-10}  \cmidrule{12-15} 
& PSNR $\uparrow$ & SSIM $\uparrow$ & LPIPS $\downarrow$ & \FLIP $\downarrow$ & 
& PSNR $\uparrow$ & SSIM $\uparrow$ & LPIPS $\downarrow$ & \FLIP $\downarrow$  &
& PSNR $\uparrow$ & SSIM $\uparrow$ & LPIPS $\downarrow$ & \FLIP $\downarrow$ \\
\midrule
StereoMag-32 & 24.45 & 0.76 & 0.107 & 0.17  && 31.40 & 0.93 & 0.031 & 0.10  &&  20.67 & 0.65 & 0.132 & 0.24  \\
StereoMag-8 & 23.00 & 0.69 & 0.126 & 0.21  && 27.76 & 0.90 & 0.044 & 0.17 && 19.13 & 0.55 & 0.152 & 0.29  \\
StereoMag-P8 & 22.31 & 0.66 & 0.209 & 0.22 && 22.00 & 0.69 & 0.160 & 0.24 && 20.11 & 0.63 & 0.156 & 0.27  \\
StereoMag-P4 & 23.69 & 0.74 & 0.137 & 0.20  && 28.06 & 0.89 & 0.066 & 0.15 &&  20.29 & 0.64 & 0.150 & 0.26 \\
IBRNet & 23.82 & 0.71 & 0.188 & 0.17  && 30.26 & 0.90 & 0.058 & 0.10  && 21.19 & 0.67 & 0.207 & 0.22  \\
\modelname-8 & 25.54 & 0.79 & 0.113 & \bf{0.14}  && 31.52 & 0.92 & 0.027 & 0.10 &&  21.58 & 0.69 & 0.149 & 0.21  \\
\modelname-4 & \bf{25.95} & \bf{0.81} & \bf{0.096} & \bf{0.14}  &&  \bf{32.61} & \bf{0.94} & \bf{0.026} & \bf{0.08} &&  \bf{22.19} & \bf{0.73} & \bf{0.125} & \bf{0.20}  \\  
\modelname-2 & 25.28 & 0.78 & 0.102 & \bf{0.14}  && 31.29 & 0.92 & 0.025 & 0.09 &&  20.78 & 0.66 & 0.141 & 0.22  \\
\bottomrule
\end{tabular}
}
\end{minipage}
\caption{Results of the evaluation on SWORD, RealEstate10K~\cite{stereomag}, and  LLFF datasets~\cite{llff}. 
For the latter dataset, models were trained on SWORD.
All metrics are computed on central crops of synthesized novel views. 
Our approach outperforms all baselines on these datasets, although it contains fewer layers in the scene proxy. 
In particular, the \modelname method surpasses IBRNet despite the fact that the latter was trained on $80\%$ of LLFF scenes in a multiview setting.
The digit after the type of the model denotes the number of layers in the estimated geometry. Suffix $P$ stands for the model after the applied postprocessing.
}
\label{tab:main_scores}
\vspace{-8pt}
\end{table*}

\subsection{Mesh texturing}
The second stage of the inference process completes the reconstruction of the scene by inferring the color and opacity textures of the layered mesh. The process is similar to~\cite{stereomag} and follow-up works with some important modifications. Most importantly, we consider non-planar/non-spherical layers predicted by the previous stage. We thus `unproject' the side view onto each of the $L$ layers, and then sample each of those reprojections at the $H\times{}W$ resolution. We employ the \texttt{nvdiffrast} differentiable renderer~\cite{Laine2020diffrast} to make the rasterization process differentiable \wrt the layered mesh geometry.
The reference view sampled at the same resolution is concatenated, resulting in a $H\times{}W\times{}(3L+3)$-sized tensor.

This tensor is then processed by \textit{coloring network} $F_c$ which aims to infer the color and opacity values for the mesh layers. 
Ultimately, our goal is predict the RGB and alpha values for each pixel in each layer. Previously, the authors of~\cite{stereomag} observed that predicting the RGB color indirectly produces better results. That is, they predicted a single ``background'' RGB image of size $H \times W \times 3$ and a layer-specific tensor of size $H \times W \times L \times 2$ that provides  blending weights for the linear combination of the reference view with the background. We confirm their findings. We have further observed that in our case even better results can be obtained by predicting an additional mixture weight tensor of size $H \times W \times L$ that contains blending weights for the side view unprojected to each layer.  

Summarizing, within our texture prediction scheme, the network $F_c$ thus produces a tensor of size $H \times W \times \left(3L+3\right)$ with the last three channels corresponding to the background image, and the remaining channels contain the mixture weights for the \textbf{R}eference view, the unprojected \textbf{S}ide view, and the \textbf{B}ack\textbf{g}round image. We refer to this scheme as \textit{RSBg}, and it is default in our experiments. In the ablation study, we further compare it with the scheme employed in~\cite{stereomag}, where only reference and background images are blended into the texture (denoted \textit{RBg}), and with the scheme that predicts RGB colors directly (denoted RAW).

In all cases, in addition to the RGB values, the network $F_c$ also predicts a tensor of shape $H \times W \times L$ containing the opacity (alpha) values for each layer. Note that the texturing scheme is able to eliminate redundant layers by setting their opacity values to zero.
\cref{fig:layers_slices} provides such examples, where some layers were made transparent by the texturing network. 

The architecture of $F_c$ is borrowed from~\cite{stereomag} (except for different shapes of the output tensors). Thus, it is a 2D-convolutional UNet-like net with dilated convolutions in the bottleneck. For completeness, we detail this architecture in the Supplementary material (\ifArxivVersion\cref{sec:architectures}\else Sec.~S1\fi).

\topic{Rendering.} To render a novel view, we project the mesh layers according to the desired pose of the camera while composing them using the \textit{compose-over} operator~\cite{compose_over}.
\cref{fig:magnification} demonstrates novel views synthesized with the proposed pipeline.

\subsection{Learning}
\label{subsec:learning}
We learn the parameters of the geometry network $F_g$ and the coloring network $F_c$ from datasets of short videos of static scenes, for which camera pose sequences have been estimated using structure-from-motion~\cite{colmap}. 
Overall, the training is performed by minimizing the weighted combination of losses discussed below.

\topic{Image-based losses.}
Similarly to previous work, \eg \cite{stereomag}, the main training loss comes from image supervision. For example, at each training step, we sample the image triplet $(I_s,I_r,I_n)$ containing the side view $I_s$, the reference view $I_r$ and the novel (hold-out) view $I_n$ from a training video.  Given the current network parameters, we estimate the scene geometry and textures from $(I_s,I_r)$ and then project the resulting representation to the $I_n$ resulting in the predicted image $\hat{I}_n$. We then compute the perceptual~\cite{johnson2016perceptual} and the $L_1$ losses between $I_n$ and $\hat{I}_n$ and backpropagate them through the networks $F_g$ and $F_c$.

\topic{Regularization losses.} 
As was explained above, BI and SA schemes of depth prediction cannot guarantee the ordering of layers.
Therefore, we apply a simple hinge loss with zero margin to layers with neighbor indices to ensure that they are predicted in front-to-back order: $L_{ord} = \sum_{j=0}^{L-1} \max \left(0, \hat{d}_{j} - \hat{d}_{j+1}\right).$ 
Additionally, we regularize the geometry of the layers by imposing the total variation (TV) loss on the depths of each layer (the total variation is computed for each of the $L$ maps encoding the depths). 

\topic{Adversarial loss.} 
While image-based and geometric losses suffice to obtain the plausible quality of novel view generation for RSBg and RBg coloring schemes (see \cref{sec:experiments} for metrics), we did not manage to obtain satisfactory results with RAW scheme without adversarial learning.
Specifically, we impose adversarial loss~\cite{goodfellow2014generative} only for the RAW scheme on the predicted images $\hat{I}_n$. The main goal of adversarial loss is to reduce unnatural artefacts such as ghosting and duplications.  To regularize the discriminator, $R_1$ penalty~\cite{mescheder_which_2018} is applied. We stress that adversarial loss is only needed for RAW prediction and is not used in our default configuration.

\begin{table}[t]
\centering
\begin{minipage}{\columnwidth}
\centering
\resizebox{0.9\textwidth}{!}{
\begin{tabular}%
{llcccc}
\toprule
\multirowcell{2}[0pt][l]{Depth \\estimation} & \multirowcell{2}[0pt][l]{Texturing \\scheme} & \multirowcell{2}{PSNR $\uparrow$} & \multirowcell{2}{SSIM $\uparrow$} & \multirowcell{2}{LPIPS $\downarrow$} & \multirowcell{2}{\FLIP $\downarrow$} \\%
&&&&& \\%
\midrule
BI & RSBg & \bf{25.95} & \bf{0.81} & \bf{0.096} & \bf{0.14}     \\
BI & RBg & 24.96 & 0.77 & 0.111 & 0.15  \\
BI & RAW & 24.90 & 0.77 & 0.099 & 0.15  \\
SA & RSBg & 24.66 & 0.76 & 0.121 & 0.16 \\
SA & RAW & 24.30 & 0.75 & 0.099 & 0.15 \\
GC & RSBg & 25.24 & 0.77 & 0.115 & 0.15  \\
GC & RAW & 24.90 & 0.77 & 0.107 & 0.15    \\
\bottomrule
\end{tabular}
}
\end{minipage}
\caption{Evaluation of \modelname configs. 
    Top row represents a model chosen as a default one.
    Details of the predicting schemes are discussed in \cref{sec:method}.
    All the configurations were trained with $P=32$ planes and $L=4$ layers.
    }
\label{tab:ablation}
\vspace{-6pt}
\end{table}

\begin{table}[t]
\centering
\resizebox{0.8\columnwidth}{!}{
\begin{tabular}{llcc}
\toprule
Dataset & Baseline & Our score, \% &  $p$-value \\%
\midrule
\multirowcell{2}[0pt][l]{SWORD} & StereoMag-32         & $61$ & $<0.001$ \\%
& IBRNet       & $81$ & $<0.001$ \\%
\midrule
\multirowcell{2}[0pt][l]{LLFF} & StereoMag-32         & $54$ & $<0.001$ \\%
& IBRNet       & $70$ & $<0.001$ \\%
\bottomrule
\end{tabular}
}
\caption{User study results. 
    The 3rd column contains the ratio of users who selected the output of our model (\modelname-4) as more realistic in side-by-side comparison. 
}
\vspace{-6pt}
\label{tab:user_study}
\end{table}

\begin{table}[t]
\centering
\begin{minipage}{\linewidth}

\resizebox{0.95\textwidth}{!}{
\begin{tabular}{llcccc}
\toprule
Transfer & Model & PSNR $\uparrow$ & SSIM $\uparrow$ & LPIPS $\downarrow$ & \FLIP $\downarrow$ \\
\midrule
\multirow{2}{*}{(R) $\rightarrow$ (S)} & StereoMag-32 & 24.45 & 0.76 & 0.107 & 0.17  \\
& \modelname-4  & 25.47 & 0.79 & 0.098 & 0.14 \\
\cmidrule{2-6}
\multirow{2}{*}{(S) $\rightarrow$ (R)} & StereoMag-32 & 31.40 & 0.93 & 0.031 & 0.10  \\
& \modelname-4 & 31.91 & 0.93 & 0.029 & 0.09  \\
\midrule
\multirow{2}{*}{(R) $\rightarrow$ (L)} & StereoMag-32 &  20.31 & 0.62 &  0.129 &  0.23 \\
& \modelname-4 & 21.52 & 0.69 & 0.133 & 0.21  \\
\cmidrule{2-6}
\multirow{2}{*}{(S) $\rightarrow$ (L)} & StereoMag-32  & 20.67 & 0.65 & 0.132 & 0.24 \\
& \modelname-4 & 22.19 & 0.73 & 0.125 & 0.20  \\
\bottomrule
\end{tabular}
}
\end{minipage}
\caption{Cross-dataset generalization. We evaluate models on RealEstate10k (R), SWORD (S) and LLFF (L) datasets. 
Notaion (X) $\to$ (Y) stands for a model, trained on dataset X and being evaluated on Y.
Generally, our approach is on par or more robust to the dataset shift, while having a more compact representation. Evaluation on hold-out LLFF dataset also shows the benefit of training on the proposed SWORD dataset (compared to RealEstate10k).
}
\label{tab:generalization_re_sword}
\vspace{-10pt}
\end{table}

\section{Experiments}
\label{sec:experiments}

\subsection{Datasets}
We consider the RealEstate10k dataset and the Local Lightfield Fusion (LLFF) dataset introduced in previous works, while also proposing a new dataset. The details of the three datasets are provided below.

\topic{RealEstate10k dataset.} Following prior works~\cite{stereomag,synsin,3d_layered_inpainting}, we evaluate our approach on the subset of \emph{RealEstate10k}~\cite{stereomag} dataset containing consecutive frames from real estate videos with camera parameters. The subset used in our experiments consists of $10,000$  scenes for training and $7,700$ scenes for test purposes. The RealEstate10k dataset serves as the most popular benchmark for novel view synthesis pipelines. Despite the relatively large size, the diversity of scenes in the dataset is limited. The dataset is predominantly indoor and also does not contain enough scenes with central objects. Consequently, models trained on RealEstate10k generalize poorly to outdoor scenes or scenes with large close-by objects~\cite{stereomag,3d_layered_inpainting}.

\topic{SWORD dataset.} To evaluate both our and prior methods on more diverse data, we have collected a new dataset, which we call \emph{`Scenes With Occluded Regions' Dataset} (SWORD). The new dataset contains around $1,500$ train video and $290$ test videos, with 50 frames per video on average. The dataset was obtained after processing the manually captured video sequences of static real-life urban scenes. The processing pipeline was the same as described in~\cite{stereomag}.

The main property of the dataset is the abundance of close objects and, consequently, the larger prevalence of occlusions.
To prove this quantitatively,  we calculate the occlusion areas, that is, areas of those regions of the novel frames that are occluded in the reference frames. To obtain masks for such regions, the off-the-shelf optical flow estimator~\cite{raft} is employed. The complete procedure for getting the occlusion masks and the examples of those masks are provided in Supplementary (\ifArxivVersion\cref{sec:occlusion_masks}\else Sec.~S5\fi). According to this heuristic, the mean area of occluded image parts for SWORD is approximately five times larger than for RealEstate10k data (14\% vs 3\% respectively). This rationalizes the collection and usage of SWORD and explains that SWORD allows training more powerful models despite being of smaller size.

\topic{LLFF dataset.} LLFF dataset is another popular dataset with central objects that was released by the authors of the paper on Local Light Field Fusion~\cite{llff}. It is too small to train a network on it (40 scenes), and we use these data for evaluation goals only to test the models trained on the other two datasets.

\subsection{Evaluation details}

\topic{Compared approaches.} We use the system described in~\cite{stereomag} as our main baseline and refer to it as \emph{\baselinestereomag}. By default, \baselinestereomag uses 32 regularly spaced fronto-parallel planes (with uniformly spaced inverse depth), for which color and transparency textures are estimated by a deep network operating on a plane sweep volume. The obtained representation is known  as ``multi-plane images'' (MPI). The original system uses this plane-based geometry for final renderings. We refer to this baseline as \textit{\baselinestereomag-32} or omit the number of planes for brevity if it is equal to the default.

Additionally, we have evaluated variants of the \baselinestereomag (denoted as \textit{\baselinestereomag-P8} and \textit{-P4}) that coalesce the 32 planes into 8 or 4  non-planar meshes respectively. The coalescence procedure is detailed in the Supplementary (\ifArxivVersion\cref{sec:postprocessing}\else Sec.~S4\fi) and is very similar to the one proposed in~\cite{immersive_lf_video}. Finally,  we trained a variant of \baselinestereomag with eight planes (\textit{\baselinestereomag-8}). 
We stress that while \baselinestereomag system was proposed some time ago, based on the comparison in the recent work~\cite[Appendix A]{3d_layered_inpainting}, it remains state-of-the-art for two image inputs.

We also consdider the more recent IBRNet~\cite{ibrnet} system trained to model the radiance field of the scene by blending features of the source images.
Unlike StereoMag, this approach has no restrictions on the number of input frames, although it requires a very significant amount of computation to generate each view. 
Moreover, as the authors have shown, IBRNet shows its best quality after fine-tuning to the new scene under consideration.
For evaluation, we used the implementation and checkpoints of the network, provided by the authors, who used 80\% of LLFF dataset for training among other data. Despite this, we compared our approach with this method on all data (including LLFF). We also tried to retrain IBRNet on the SWORD dataset (in the setting of two input images). This, however, led to considerably worse performance, so we stick with the authors' provided variant. We also note that test-time fine-tuning of IBRNet is not possible with two view inputs.  

We trained different variations of our model with $L \in \{2, 4, 8\}$ layers obtained from $P=32$ planes of PSV, unless another number is specified.
All models were trained for $500,000$ iterations with batch size $4$ on a single NVIDIA P40 GPU. The training time is not particularly different for the listed variants of the model.
For our approach, we set the following weights for the losses described above: $1$ for $L_1$ loss, $10$ for perceptual loss, $5$ for TV regularization, and $2$ for ordering loss. The RAW scheme for RGB prediction requires a careful tuning of parameters, and we report its results for the configuration with adversarial and feature matching losses with weights set to $5$, while the discriminator gradient was penalized every $16$-th step with the weight of $R_1$ penalty equal to $0.0001$. 
Most experiments were conducted at the resolution of $256$ on the smallest side.

\topic{Metrics.} 
We follow the standard evaluation process for the novel view task and measure how similar the synthesized view is to the ground-true image. Therefore, we compute the peak signal-to-noise ratio (PSNR), structural (SSIM), and perceptual (LPIPS~\cite{lpips}) similarity, as well as the recently introduced \FLIP metric~\cite{Andersson2020} between the obtained rendering and the ground truth.
Artifacts in areas near the image boundary are similar both for planes and layers, and we exclude those regions from consideration by computing metrics over the central crops. %

Finally, to measure the plausibility of rendered images, we perform the study of human preference on a crowdsourcing platform. The evaluation protocol was as follows: The assessors were shown two short videos with the virtual camera moving along the predefined trajectory in the same scene from SWORD (validation subset) or LLFF: one video was obtained using the baseline model, and another one was produced with our approach.
We asked the users which of the two videos looked more realistic to them.
In total, we generated 280 pairs of videos (120 from LLFF and 160 from SWORD scenes), and twenty different workers assessed each pair.

\begin{figure}[t]
    \centering
    \includegraphics[width=0.9\columnwidth]{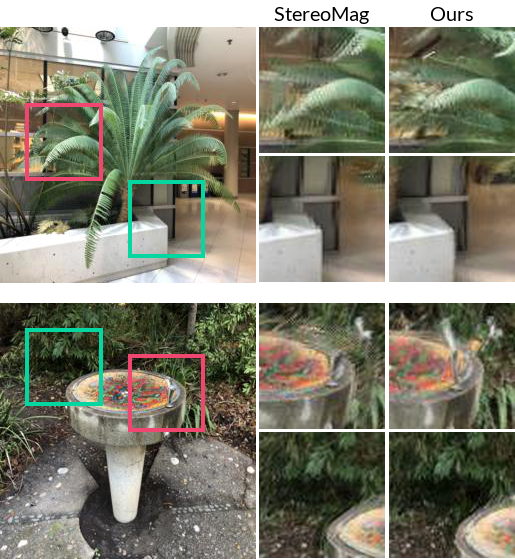}
    \caption{
        Comparison on  challenging scenes from the LLFF~\cite{llff}. The leftmost column shows the ground truth, two other columns demonstrate patches of a novel view obtained with  StereoMag-32 and our system respectively.
        In the cutout StereoMag results, small translations of a camera from the reference pose reveal discontinuities in the approximate scene geometry leading to ghosting artefacts. In our case, thanks to the scene-adapted geometry, ghosting is not so apparent.
    }
    \label{fig:crop_ours}
\vspace{-8pt}
\end{figure}

\subsection{Main results}
\label{subsec:ablation}

\topic{Ablation results.} In \cref{tab:ablation} we present the relative performance of several schemes of depth estimation (denoted GC, SA and BI) and mesh texturing (RSBg, RBg, RAW). For this ablation, all systems were trained and evaluated on the SWORD dataset. As the results show, the best metrics are obtained with a combination of BI + RBSBg methods. Therefore, we choose this model as our default one and refer to it as just \modelname model. This pipeline is used in further experiments, unless another configuration is explicitly specified.

\topic{Comparison with prior art.} The main results are reported in \cref{tab:main_scores}.
Here, due to the relatively small size of the validation part of the both of SWORD and LLFF datasets, we sampled multiple triplets (the reference, side, and novel cameras) for each scene to get a more accurate estimation of the score.
We selected the model with $L=4$ layers as our main variant because it performs better on the hold-out LLFF data.
It consistently outperforms the baseline StereoMag-32 model  according to the considered metrics, while containing significantly less layers.

Post hoc coalescing of StereoMag representations into non-planar layers worsened the results (this finding is consistent with one reported in \cite{immersive_lf_video}). Finally, the eight-planar structure is consistently worse than the 32-plane one.

At the same time, the results of our model with four layers are even better than when using more (eight) layers. Furthermore, a configuration with just two layers remains competitive (better than eight-layer configuration in some metrics and better than StereoMag-32 in most metrics). 

Notably, there is a large gap between our configuration with four layers and the StereoMag method with geometry merged into four layers (StereoMag-P4). This emphasizes the benefit of the end-to-end training used by our method.
As showcased in \cref{fig:layers_slices}, the novel approach approximates the scene geometry in a reasonable way even with just two layers and, vice versa, is able to `zero out' the redundant layers.
We attribute the superiority of our method over StereoMag-P  to the proposed end-to-end training procedure.

We show the percentage of times users prefer each method in \cref{tab:user_study}.
One of our qualitative improvements is illustrated in \cref{fig:crop_ours}: the deformable layers successfully overcome the ``ghost'' edge artifacts, occasionally observed in the case of rigid planes. 

Also, we conducted a separate study of the model's sensitivity to the group size when estimating the scene geometry (\ie the ratio of the number of planes in PSV $P$ to the number of layers $L$).
In summary, quality does not change dramatically under variation in group size; see Supplementary (\ifArxivVersion\cref{sec:additional_result}\else Sec.~S2\fi) for numerical details.

In the supplementary video, we show the results on a wide range of photos from different datasets. The video contains a comparison with StereoMag and IBRNet that demonstrates that our approach produces less blurry details while having a similar quality of estimated geometry. We encourage the reader to watch the supplementary video.

\topic{Cross-dataset evaluation.} 
As mentioned above, SWORD contains mostly outdoor scenes with  a central object, which is similar in nature to the LLFF dataset. 
It is the main reason why we observed a pretty good quality of the model trained on SWORD and evaluated on LLFF (the rightmost part of \cref{tab:main_scores}).
We have also investigated a more challenging  setting: the performance of methods in the cross-dataset setting is reported in \cref{tab:generalization_re_sword}: we cross-evaluate our and baseline models on RealEstate10k and SWORD datasets that are rather different.

\topic{Timings.} The representations produced by our method are well suited for rendering within mobile photography applications.
Thus, on Samsung Galaxy S20 (Mali-G77 GPU), rendering our representations at $512\times{}256$ resolution runs at about 180 frames per second. 
Furthermore, our representations can be quickly created from new stereo pairs (our current unoptimized inference takes 0.19 seconds on an NVidia P40 GPU).

\section{Summary and discussion}
In this work, we proposed an end-to-end pipeline that recovers the geometry of the scene from an input stereo pair using a fixed number of semitransparent layers.
Despite using fewer layers (4 layers \vs 32 planes for the baseline StereoMag model), our approach demonstrated superior quality in terms of commonly used metrics for the novel view synthesis problem, as well as human evaluation.
Unlike the StereoMag system, the quality of which heavily depends on the number of planes, our method has reached better scores  while being robust to reducing the number of layers.
We have verified that the proposed method can be trained on multiple datasets and generalizes well to unseen data.
The resulting mesh geometry can be effectively rendered using standard graphics engines, making the approach attractive for mobile 3D photography.

Additionally, we presented a new challenging SWORD dataset, which contains cluttered scenes with heavily occluded regions. Even though SWORD consists of fewer scenes than the popular RealEstate10K dataset, systems trained on SWORD are likely to generalize better to other datasets, \eg the LLFF dataset.

\cleardoublepage

\appendix
\renewcommand{\thetable}{S\arabic{table}}
\renewcommand{\thefigure}{S\arabic{figure}}
\renewcommand{\theequation}{S\arabic{equation}}

\setcounter{figure}{0}
\setcounter{table}{0}
\setcounter{equation}{0}

\section{Network architectures}
\label{sec:architectures}
\topic{Geometry network $F_g$.} 
The architecture of our depth estimator resembles the network from SynSin~\cite{synsin}.
It takes the plane sweep volume (PSV) as its input and returns `opacities' for each of the $P$ regular planes, that are used to construct deformable layers.
Each block sequentially applies a convolution, layer normalization and LeakyReLU to the input tensor.
We apply spectral normalization~\cite{miyato_spectral_2018} to the convolution kernel weights. 
Other details are given in \cref{tab:depth_network_architecture}.

\begin{table}[t]
\centering
\resizebox{\columnwidth}{!}{\begin{tabular}{lcccccc}
\toprule
Block    & K & S  & D & P & C       & Input \tabularnewline
\midrule
Conv1\_1 & 4 & 2  & 1 & 1 &  32     & PSV \tabularnewline
Conv1\_2 & 4 & 2  & 1 & 1 &  64     & Conv1\_1 \tabularnewline
Conv1\_3 & 4 & 2  & 1 & 1 &  128    & Conv1\_2 \tabularnewline
Conv2\_1 & 4 & 2  & 1 & 1 &  256    & Conv1\_3 \tabularnewline
Conv2\_2 & 4 & 2  & 1 & 1 &  256    & Conv2\_1 \tabularnewline
Conv2\_3 & 4 & 2  & 1 & 1 &  256    & Conv2\_2 \tabularnewline
Conv2\_4 & 4 & 2  & 1 & 1 &  256    & Conv2\_3 \tabularnewline
Conv2\_5 & 4 & 2  & 1 & 1 &  256    & Conv2\_4 \tabularnewline
\midrule
Conv3\_1 & 3 & 1  & 1 & 1 &  256    & Conv2\_5$\uparrow$ \tabularnewline
Conv3\_2 & 3 & 1  & 1 & 1 &  256    & \texttt{concat}[Conv3\_1, Conv2\_4]$\uparrow$ \tabularnewline
Conv3\_3 & 3 & 1  & 1 & 1 &  256    & \texttt{concat}[Conv3\_2, Conv2\_3]$\uparrow$ \tabularnewline
Conv3\_4 & 3 & 1  & 1 & 1 &  256    & \texttt{concat}[Conv3\_3, Conv2\_2]$\uparrow$  \tabularnewline
Conv4\_1 & 3 & 1  & 1 & 1 &  128    & \texttt{concat}[Conv3\_4, Conv2\_1]$\uparrow$ \tabularnewline
Conv4\_2 & 3 & 1  & 1 & 1 &  64     & \texttt{concat}[Conv4\_1, Conv1\_3]$\uparrow$ \tabularnewline
Conv4\_3 & 3 & 1  & 1 & 1 &  32     & \texttt{concat}[Conv4\_2, Conv1\_2]$\uparrow$ \tabularnewline
Conv4\_4 & 3 & 1  & 1 & 1 &  $P$      & \texttt{concat}[Conv4\_3, Conv1\_1]$\uparrow$ \tabularnewline
\bottomrule
\end{tabular}}
\caption{Architecture of the geometry network $F_g$ for BI parameterization. \textit{K} is the kernel size, \textit{S} -- stride, \textit{D} -- dilation, \textit{P} -- padding, \textit{C} -- the number of output channels for each layer, and \textit{input} denotes the input source of each layer. Up-arrow $\uparrow$ denotes the 2x bilinear upscaling operation.
}
\label{tab:depth_network_architecture}
\end{table}

\topic{Coloring network $F_c$.}
The architecture of the coloring network is inspired by the one described in StereoMag paper~\cite{stereomag}. 
Each block consists of a convolution, layer normalization, and ReLU unit (except for the final block).
Detailed parameters for RSBg scheme are provided in \cref{tab:coloring_network_architecture}.

\begin{table}[t]
\centering
\resizebox{\columnwidth}{!}{\begin{tabular}{lcccccc}
\toprule
Block &         K & S  & D & P & C    & Input \tabularnewline
\midrule
Conv1\_1      & 3 & 1  & 1 & 1 &   64     & deformed PSV \tabularnewline
Conv1\_2      & 3 & 2  & 1 & 1 &  128     & Conv1\_1 \tabularnewline
Conv2\_1      & 3 & 1  & 1 & 1 &  128     & Conv1\_2 \tabularnewline
Conv2\_2      & 3 & 2  & 1 & 1 &  256     & Conv2\_1 \tabularnewline
Conv3\_1      & 3 & 1  & 1 & 1 &  256     & Conv2\_2 \tabularnewline
Conv3\_2      & 3 & 1  & 1 & 1 &  512     & Conv3\_1 \tabularnewline
Conv3\_3      & 3 & 2  & 1 & 1 &  512     & Conv3\_2 \tabularnewline
\midrule
Conv4\_1      & 3 & 1  & 2 & 2 &  512     & Conv3\_3 \tabularnewline
Conv4\_2      & 3 & 1  & 2 & 2 &  512     & Conv4\_1 \tabularnewline
Conv4\_3      & 3 & 1  & 2 & 2 &  512     & Conv4\_2 \tabularnewline
\midrule
TransConv5\_1 & 4 &  2 & 1 & 1 &  256     &   \texttt{concat}[Conv4\_3, Conv3\_3] \tabularnewline
TransConv5\_2 & 3 &  1 & 1 & 1 &  256     &  TransConv5\_1 \tabularnewline
TransConv5\_3 & 3 &  1 & 1 & 1 &  256     &  TransConv5\_2 \tabularnewline
TransConv6\_1 & 4 &  2 & 1 & 1 &  128     &  \texttt{concat}[TransConv5\_3, Conv2\_2] \tabularnewline
TransConv6\_2 & 3 &  1 & 1 & 1 &  128     &  TransConv6\_1 \tabularnewline
TransConv7\_1 & 4 &  2 & 1 & 1 &   64     &  \texttt{concat}[TransConv6\_2, Conv1\_2] \tabularnewline
TransConv7\_2 & 3 &  1 & 1 & 1 &   64     & TransConv7\_1 \tabularnewline
Conv7\_3      & 1 &  1 & 1 & 0 &   $4L+3$ &      TransConv7\_2 \tabularnewline
\bottomrule
\end{tabular}}
\caption{Architecture of the coloring network $F_c$ for the RSBg parameterization. \textit{K} is the kernel size, \textit{S} -- stride, \textit{D} -- dilation, \textit{P} -- padding, \textit{C} -- the number of output channels for each layer, and \textit{input} denotes the input source of each layer.
}
\label{tab:coloring_network_architecture}
\end{table}

\section{Additional results}
\label{sec:additional_result}
\topic{Scaling to hi-res.}
To investigate the scaling properties of our \modelname model, we additionally compared it with the baselines on high-resolution versions of datasets, described in the main text.
\cref{tab:hires_scores} presents the results of the trained network, applied to higher resolution in a fully convolutional manner.
It outperforms StereoMag operating in the same regime by a significant margin. 
Additionally, we compare the quality with the original IBRNet. This model achieves the best PSNR value and simultaneously the worst LPIPS. This is caused by inconsistency in the generated frames. Please see examples of such behaviour in the supplementary video.

Besides that, we conducted a user study on 80 scenes from SWORD (with resolution of $512 \times 1024$), 60 scenes from RealEstate10k ($576 \times 1024$) and 80 scenes (40 unique) from LLFF data ($512 \times 512$). All scenes and input views are randomly sampled from the test sets. The results of this experiment are reported in \cref{tab:userstudy_high}.

\begin{figure}[t]
    \centering
    \includegraphics[width=\columnwidth]{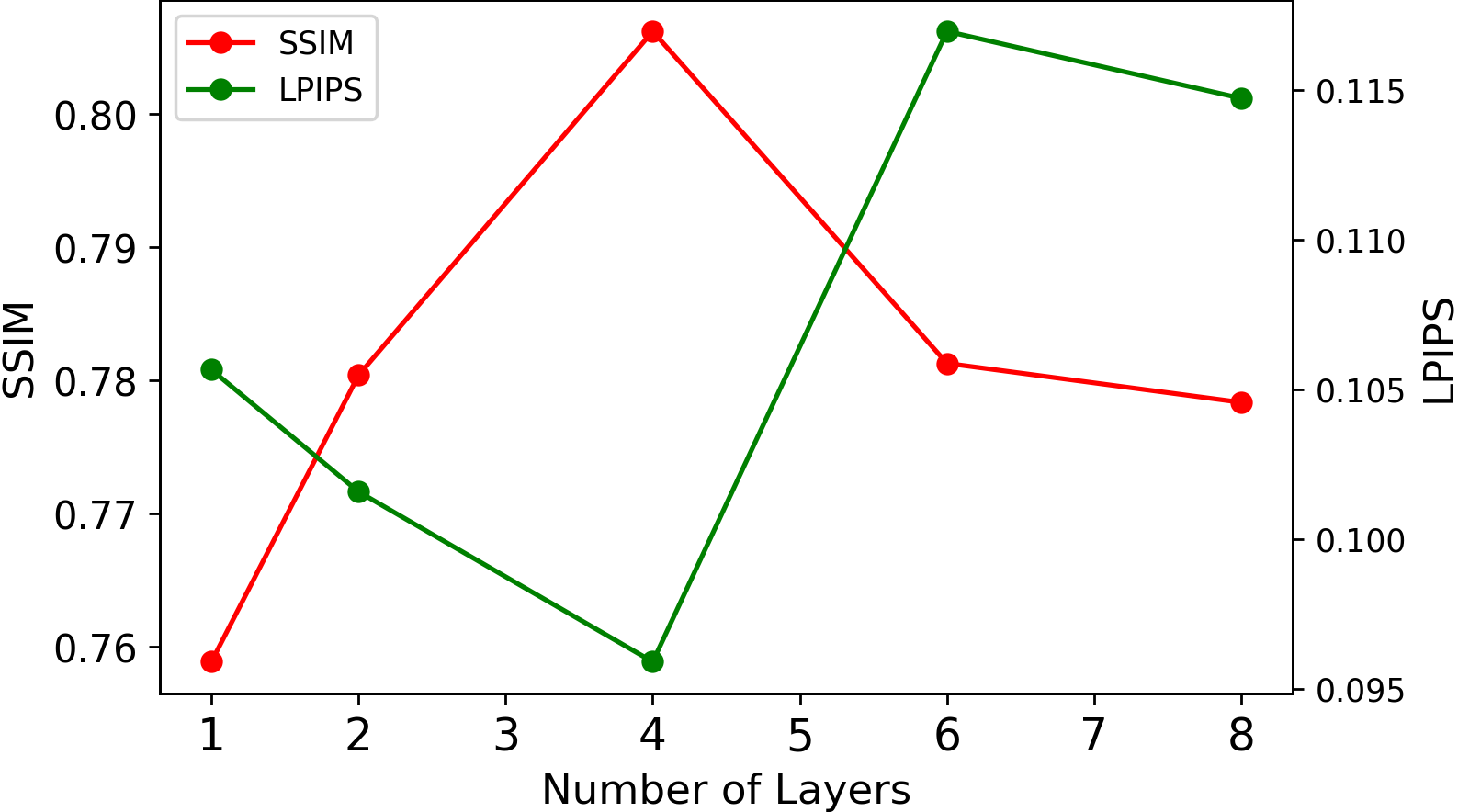}
    \caption{Performance of our system as a function of the number of layers. 
    The plot confirms the ability of our approach to represent complex scenes with just a few layers. 
    }
    \label{fig:depend_num_layers}
\end{figure}

\begin{table}[t]
\centering
\resizebox{\columnwidth}{!}{\begin{tabular}{lcccc}
 \toprule
Model & PSNR $\uparrow$ & SSIM $\uparrow$ &LPIPS $\downarrow$ & \FLIP $\downarrow$ \\
\midrule
IBRNet & \bf{27.4} & 0.67 & 0.219  & 0.27 \\
StereoMag (256$\rightarrow$512) & 23.3 &  0.65 &   0.178 & \bf{0.19}  \\
Ours (256$\rightarrow$512) & 24.2 & \bf{0.69} & \bf{0.155} & \bf{0.19} \\
\bottomrule
\end{tabular}
}
\caption{Scaling to higher resolution on SWORD dataset. We examine our model and StereoMag in a fully-convolutional regime: both were trained at resolution of $256 \times 512$ and applied for $512 \times 1024$. As in previous experiments, we used the checkpoint of IBRNet provided by the authors of the corresponding paper.
}
\label{tab:hires_scores}
\end{table}

\begin{table}[t]
\centering
\begin{tabular}{llcc}
\toprule
Dataset & Baseline & Our score, \% &  $p$-value \\%
\midrule
\multirowcell{2}[0pt][l]{SWORD} & StereoMag-32         & 55.62 & $<0.001$ \\%
& IBRNet       & 75.69 & $<0.001$ \\%
\midrule
\multirowcell{2}[0pt][l]{LLFF} & StereoMag-32         & 54.42 & $<0.001$ \\%
& IBRNet       & 50.27 & $<0.001$ \\%
\midrule
\multirowcell{2}[0pt][l]{RealEstate10k} & StereoMag-32         & 63.91 & $<0.001$ \\%
& IBRNet       & 60.74 & $<0.001$ \\%
\bottomrule
\end{tabular}
\caption{Additional user study on high-resolution images.
    The 3rd column contains the ratio of users who selected the  output of our model as more realistic under the two-alternative forced choice.
}
\label{tab:userstudy_high}
\end{table}

\topic{StereoMag with RSBg scheme.}
As was shown in the \ifArxivVersion\cref{tab:ablation} \else Tab.~2 \fi of the main text, our model trained with the RBg texturing scheme (which is the default for StereoMag) performs significantly worse than with RSBg:
LPIPS of 0.111 vs 0.096.
To demonstrate that the texturing scheme is not the most crucial part of our pipeline, we retrained StereoMag-32 model with RSBg scheme.
In particular, this modification did not improve the quality of the baseline on SWORD: SSIM of 0.77 vs 0.76, LPIPS of 0.107 vs 0.107.

\topic{Scene slices.}
\cref{fig:layers_slices_supmat} provides additional examples of the estimated geometry for different scenes.

\begin{figure*}[ht]
\centering
    \includegraphics[width=\linewidth]{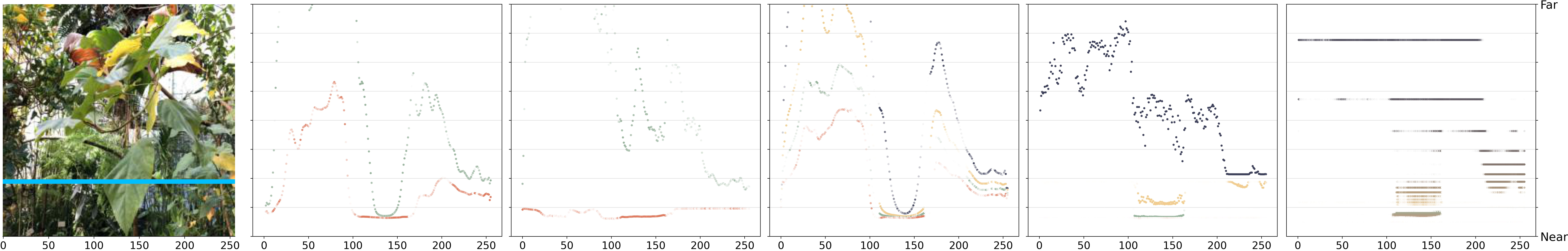}
    \includegraphics[width=\linewidth]{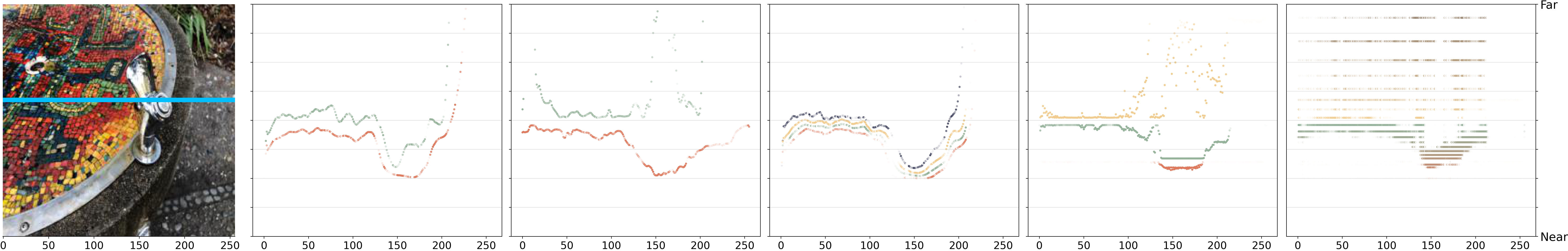}
    \includegraphics[width=\linewidth]{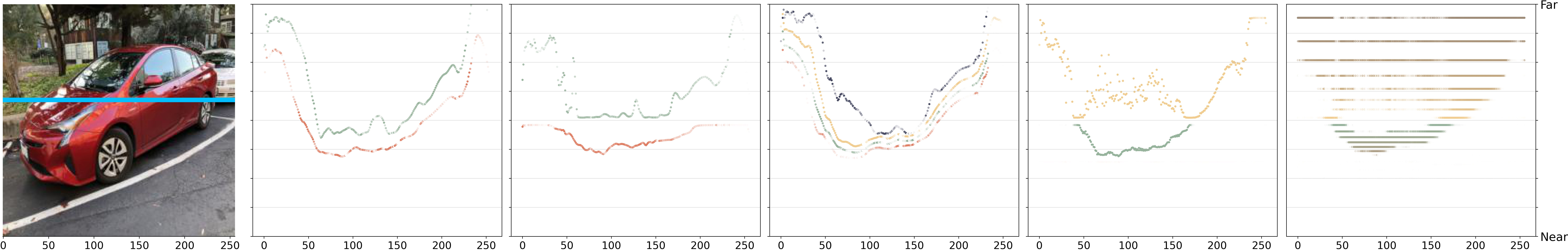}
    \includegraphics[width=\linewidth]{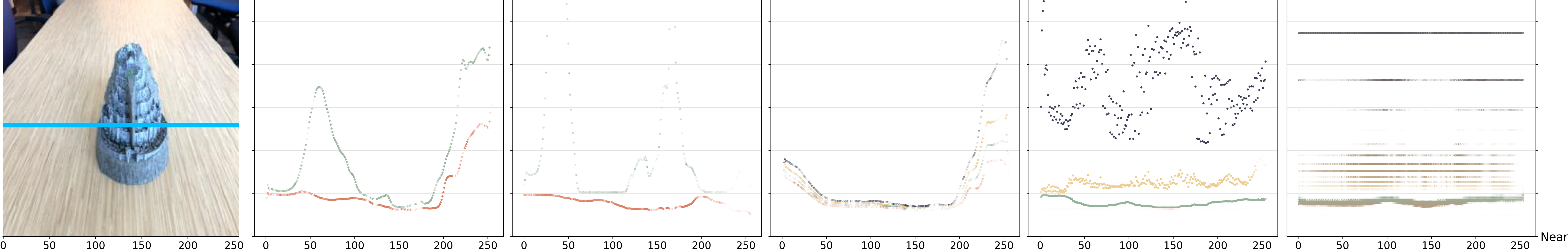}
    \includegraphics[width=\linewidth]{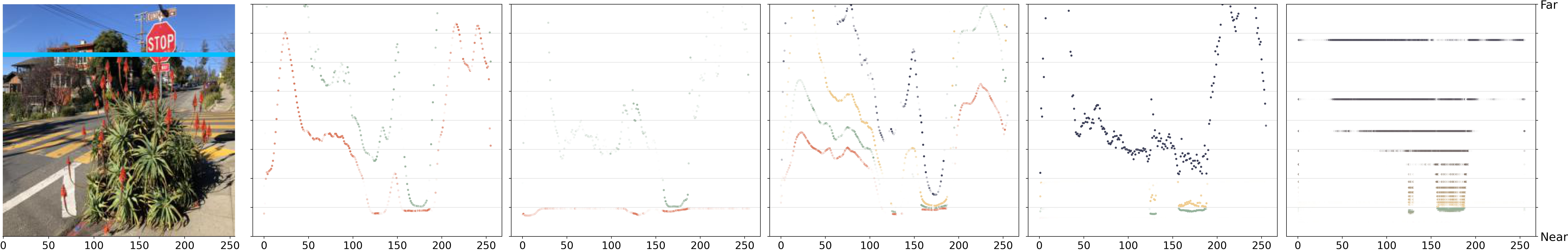}
    \includegraphics[width=\linewidth]{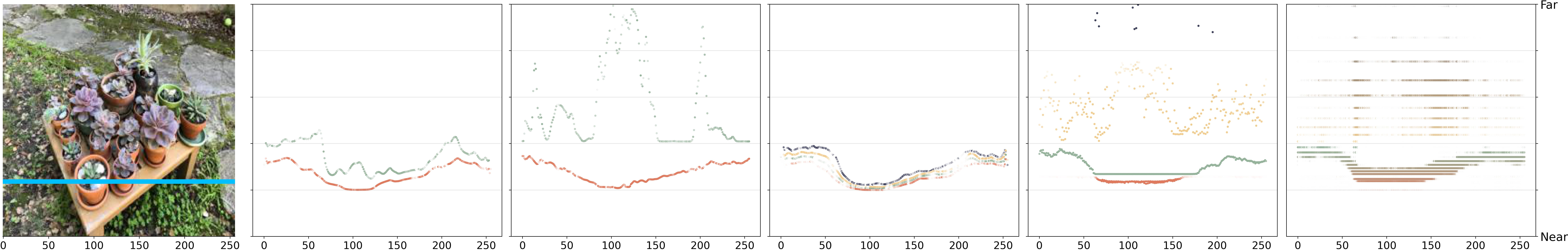}
    \includegraphics[width=\linewidth]{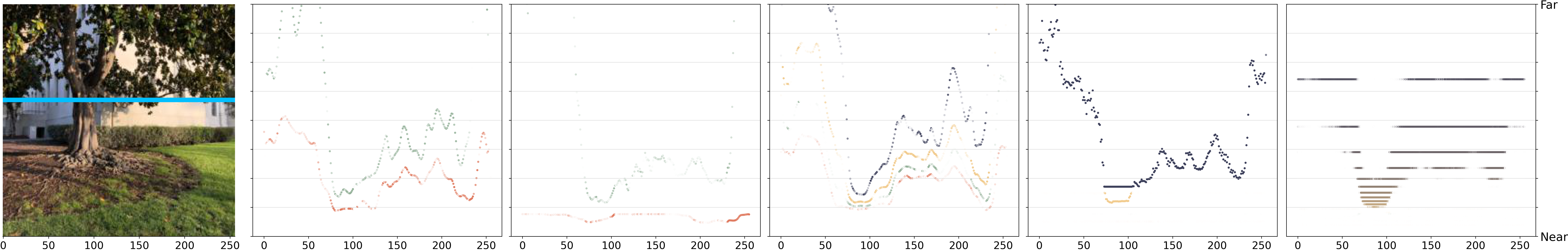}
\\ \makeatletter\@for\items:={Reference view,\modelname-2 (BI+RSBg),\modelname-2 (GC+RSBg),\modelname-4 (BI+RSBg),  StereoMag-P4,  StereoMag-32}\do{\minipage{0.16\textwidth}\centering\small \items \endminipage\hfill}\makeatother
    \vspace{-5pt}
    \caption{ 
        Additional horizontal slices (along the blue line) on  scenes from LLFF dataset. 
        Mesh vertices are shown as dots with the predicted opacity. Colors encode the layer number. The horizontal axis corresponds to the pixel coordinate, while the vertical axis stands for the vertex depth \wrt the reference camera (only the most illustrative depth range is shown). 
        Configurations of \modelname method generate scene-adaptive geometry in a more efficient way than StereoMag, resulting in more frugal geometry representation, while also obtaining better rendering quality.
    }
    \label{fig:layers_slices_supmat}
    \vspace{-10pt}
\end{figure*}

\topic{Number of layers in BI scheme.}
For MPI-based approaches, the number of planes was shown to be critical for constructing a plausible representation of the scene~\cite{pushing_bound,llff}.
To demonstrate the properties of our deformable layers, we consider the influence of the number of layers in the estimated geometry on common quality metrics.
\cref{fig:depend_num_layers} shows that the resulting performance falls as the number of layers decreases to one, proving that multi-layer structure is crucial. 
Perhaps surprisingly, the measured quality does not always grow as this number increases. We suggest that the model cannot handle the redundant geometry properly.
It is worth noting that the authors of the Worldsheet paper reported a similar effect in the single-image case~\cite{worldsheet}.

\topic{Number of layers per group in GC scheme.}
In addition to our main \emph{bounds interpolation} (BI) scheme of depth parameterization, we study the properties of the \emph{group compositing} (GC) model.
Namely, we investigate the performance of this system as a function of the number of planes in plane sweep volume during the geometry estimation step.
As \cref{tab:num_layers_depnd} shows, the resulting quality of the model does not depend on the size of the group. However, we see that if both the number of layers and the size of the group are reduced simultaneously, the quality deteriorates. And with an increase in the size of the group, there is no increase in metrics. In general, robustness to these parameters is provided by two points: the nature of the semitransparent proxy geometry, in which the alpha channel takes the main responsibility for the object structure, and the adaptive layered proxy geometry, which can bend itself under objects to depend less on the number of planes.

\begin{table}[t]
\centering
\resizebox{\columnwidth}{!}{
\begin{tabular}{ccccc}
\toprule
 \begin{tabular}[c]{@{}c@{}} Group size \\ $P / L$ \end{tabular}
    & \begin{tabular}[c]{@{}c@{}}  Number of\\planes $P$ \end{tabular} 
    &\begin{tabular}[c]{@{}c@{}}  Number of\\layers $L$ \end{tabular}  
    & LPIPS$\downarrow$ 
    & SSIM$\uparrow$
    \\
\midrule
4 & 16 & 4 & 0.129 & 0.67 \\
4 & 24 & 6 & 0.120 & 0.70 \\
4 & 32 & 8 & 0.119 & 0.70 \\
4 & 40 & 10 & 0.124 & 0.70 \\
\midrule
16 & 64 & 4 & 0.122 & 0.72 \\
32 & 64 & 2 & 0.121 & 0.71 \\
\midrule
15 & 120 & 8 & 0.119 & 0.70 \\
20 & 120 & 6 & 0.122 & 0.70 \\
30 & 120 & 4 & 0.120 & 0.70 \\
60 & 120 & 2 & 0.119 & 0.70 \\
\bottomrule
\end{tabular}
}
\caption{ Performance dependence on the number of layers and the size of the plane group for group compositing (GC) configuration. The quality in terms of SSIM and LPIPS is slightly dependent on the size of the group and the number of layers for $256 \times 256$ images.
}
\label{tab:num_layers_depnd}
\end{table}

\section{Failure cases}
\label{sec:failure_cases}
\begin{figure}[t]
    \centering
    \includegraphics[width=0.95\columnwidth]{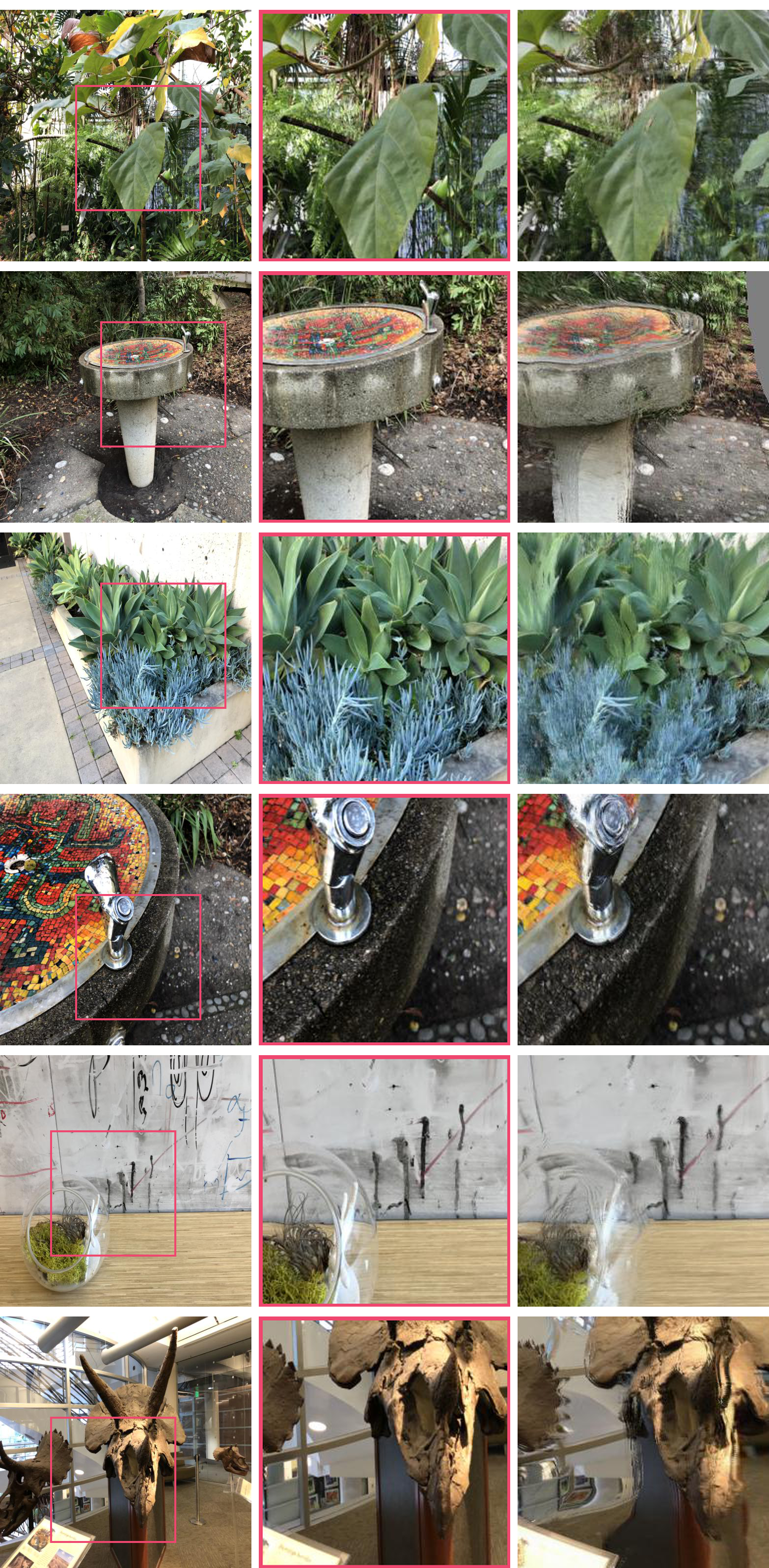}\\
    \begin{minipage}{0.95\linewidth}
    \makeatletter\@for\items:={Full frame,Ground truth,Rendered view}\do{\minipage{0.33\textwidth}\centering\small \items \endminipage\hfill}\makeatother
    \end{minipage}
    \caption{Examples of most common failures of \modelname outputs. In most cases they can be attributed to a combination of photometric scene complexity, and an unfortunate choice of the input pair.
    }
    \label{fig:failures}
    \ifArxivVersion\vspace{-12pt}\fi
\end{figure}
To demonstrate the limitations of our approach, we show typical artifacts of the method in \cref{fig:failures}.
Note that most of the drawbacks are visible only when the camera moves around the scene and are not distinguishable in randomly selected frames without temporal context.

When the baseline is magnified by a great factor, one can observe ``stretching'' faces of our layered mesh near the depth discontinuities.
We believe that this type of artifact is caused by the mesh structure of our geometry.
The ``ghost'' semitransparent textures is another common issue of the synthesized views.
One of the problems could also be attributed to inconsistent depth prediction when some pixels have minor errors in depth values, which leads to small ghostings.

\section{MPI postprocessing}
\label{sec:postprocessing}
In this section we briefly describe the postprocessing procedure that aims to merge the predicted rigid planes of StereoMag-32~\cite{stereomag} to the fewer number of deformable layers.
In our experiments, the final number of such layers equals 8, that coincides with the basic configuration of our approach.

The pipeline partially follows the one described in~\cite{immersive_lf_video}.
Firstly, we divide 32 planes into eight groups and compose-over the depth within each group on top of the furthest plane in the group.
This operation results in 8 deformable layers.
To infer the textures of those layers, we perform the second step, averaging the color $\mathbf{c}$  and transmittance $\bar{\alpha}$  of RGBA planes within each group over the set $V\left(t\right)$ of rays passing through the texel $t$.
Namely, we run the Monte Carlo ray tracing defined by the equations below,
\begin{align*}
    \log \left( \bar{\alpha}_t \right) &= \lambda^{-1} \int_{V\left( t \right)} w\left(\mathbf{r}\right) \left[ \log \left( \bar{\alpha}_\mathbf{r} \right) \right]^2 d \mathbf{r}, \\
    \mathbf{c}_t &= \lambda^{-1} \int_{V\left( t \right)} w\left(\mathbf{r}\right) \mathbf{c}_\mathbf{r} \log \left( \bar{\alpha}_\mathbf{r} \right) d \mathbf{r}, 
\end{align*}
where $\lambda$ is a normalizing constant
\begin{align*}
    \lambda =  \int_{V\left( t \right)} w\left(\mathbf{r}\right)  \log \left( \bar{\alpha}_\mathbf{r} \right)  d \mathbf{r}.
\end{align*}
The distribution of rays $\mathbf{r} \in V\left(t\right)$ is constructed as follows: the line passing through the pinhole camera and texel $t$ intersects the reference image plane at the pixel coordinate $p$. 
The coordinate $q$ is normally distributed around $p$, and the ray $\mathbf{r}$ passes from $q$ through $t$.
The weighing function $w\left(\mathbf{r}\right)$ is equal to the Gaussian density value at $q$.
Color $\mathbf{c}_\mathbf{r}$ and transmittance $\bar{\alpha}_\mathbf{r}$ values are computed with the compose-over operation along the ray $\mathbf{r}$ over the planes that belong to the same group as texel $t$ does.

\section{Occlusion masks}
\label{sec:occlusion_masks}
In this section, we describe the heuristic to create masks of occluded regions.
Examples of such masks are provided in Fig.~\ref{fig:raft_mask}.

\subsection{Cycle consistency of optical flows}
Consider two images $A$ and $B$, without loss of generality, they are assumed to be grayscale.
For the coordinates of the pixel $p$ we denote the color of this pixel in the image $A$ as $A\left[ p \right].$
The coordinate grid $G$ is such a ``image'' (two-dimensional matrix) that $\forall p \, G\left[ p \right] = p$.
We define the backward flow matrix $\overset\leftarrow{F}_{AB}$ of images $A$ and $B$ and the backward warping \texttt{backward} operation as follows
\begin{equation}
\footnotesize
    B = \texttt{backward} \left(A, \overset\leftarrow{F}_{AB} \right) \iff \forall q \, B\left[q\right] = A\left[ \overset\leftarrow{F}_{AB} \left[ q \right]\right].
\label{eq:backward_flow}
\end{equation}

Similarly, forward flow matrix $\overset\rightarrow{F}_{AB}$ and \texttt{forward} warping are defined as
\begin{equation}
\footnotesize
    B = \texttt{forward} \left(A, \overset\rightarrow{F}_{AB} \right) \iff \forall p \, A \left[p\right] = B \left[ \overset\rightarrow{F}_{AB} \left[ p \right]\right].
\label{eq:forward_flow}
\end{equation}

\begin{lemma}
For two optical flows of the same kind $F_{AB}$ and $F_{BA}$ the following cycle-consistency property holds
$$\texttt{backward} \left(F_{BA}, F_{AB}\right) = G.$$
\label{th:flows_consistency}
\end{lemma}
\begin{proof}
We assume that the pixel $p$ of the image $A$ corresponds to the pixel $q$ of the image $B$ under the warping operation.
This implies the following equations:
\begin{align}
    B\left[q\right] &= A\left[p\right], \\
    \overset\leftarrow{F}_{AB}\left[q\right] &\stackrel{(\ref{eq:backward_flow})}{=} p, \label{eq:bfq}\\
    \overset\rightarrow{F}_{AB}\left[p\right] &\stackrel{(\ref{eq:forward_flow})}{=} q. \label{eq:ffp}
\end{align}
By a symmetry argument, we also obtain
\begin{align}
    \overset\leftarrow{F}_{BA}\left[p\right] &\stackrel{(\ref{eq:bfq})}{=} q, \label{eq:bfp}\\
    \overset\rightarrow{F}_{BA}\left[q\right] &\stackrel{(\ref{eq:ffp})}{=} p. \label{eq:ffq}
\end{align}
Let $X$ be the result of warping one backward flow with another, $$X =  \texttt{backward} \left(\overset\leftarrow{F}_{BA}, \overset\leftarrow{F}_{AB} \right).$$
From the definition,
\begin{align*}
    X \left[q\right] 
    \stackrel{(\ref{eq:backward_flow})}{=} \overset\leftarrow{F}_{BA} \left[ \overset\leftarrow{F}_{AB}\left[ q \right] \right]
    \stackrel{(\ref{eq:bfq})}{=} \overset\leftarrow{F}_{BA} \left[p\right]
    \stackrel{(\ref{eq:bfp})}{=} q,
\end{align*}
therefore, $X = G$.

The case of forward flow may be considered in the same way.
Denote the result of warping with $Y$, $$Y =  \texttt{backward} \left(\overset\rightarrow{F}_{BA}, \overset\rightarrow{F}_{AB} \right).$$
The value in the pixel $p$ gives us the following
\begin{align*}
    Y\left[ p \right] 
    \stackrel{(\ref{eq:backward_flow})}{=} \overset\rightarrow{F}_{BA} \left[ \overset\rightarrow{F}_{AB}\left[ p \right] \right]
    \stackrel{(\ref{eq:ffp})}{=} \overset\rightarrow{F}_{BA} \left[ q \right]
    \stackrel{(\ref{eq:ffq})}{=} p,
\end{align*}
which leads to $Y = G$.
\end{proof}

\subsection{Estimation of occlusion masks}
We employ the pretrained optical flow estimator~\cite{raft} and compute optical flows $\hat{F}_{rn}$ and $\hat{F}_{nr}$ between the reference view $I_r$ and ground-true novel view $I_n$.
According to the lemma~\ref{th:flows_consistency}, these flows should be cycle-consistent.
However, the views do not completely correspond to each other because of the presence of occluded regions.
Therefore, the result $\hat{G}$ of warping of one flow with another $$\hat{G} = \texttt{backward}\left(\hat{F}_{rn}, \hat{F}_{nr}\right)$$ does not result in the ``ideal'' coordinate grid. 

Based on this, we treat a pixel $p$ that $|\hat{G} \left[ p \right] - p| < \epsilon$ as \emph{non-occluded} because the optical flow estimator can find the corresponding pixel in another image. 
Otherwise, we include the pixel in the occlusion mask.
The threshold $\epsilon$ is set to the size of one pixel.
As a downside, the flow estimator is very sensitive to the image borders.
To overcome this issue, we use central crops that finally contain reasonable masks.

\begin{figure}[t]
    \centering
    \includegraphics[width=0.9\columnwidth]{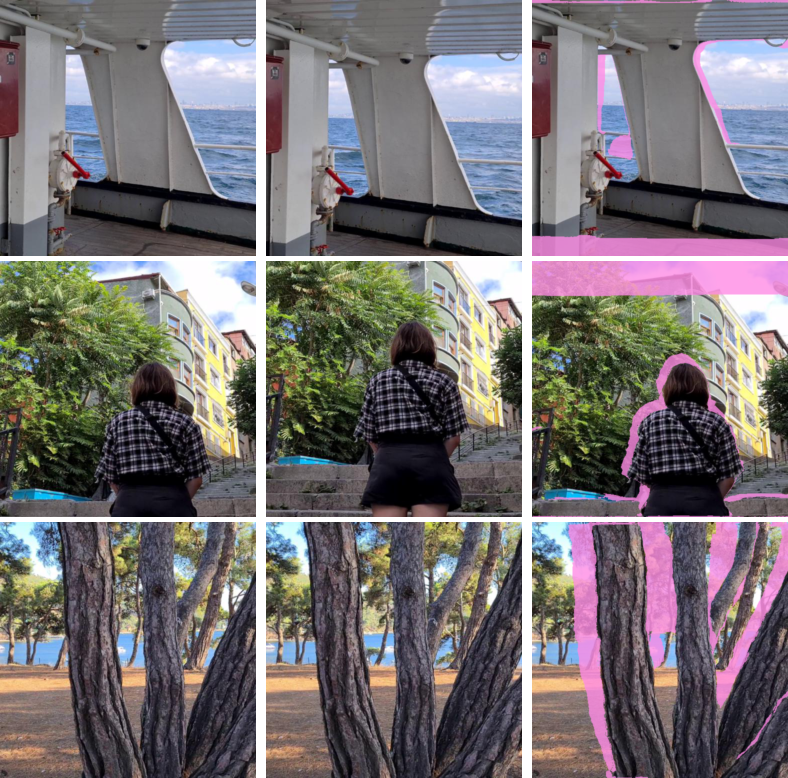}
    \caption{Occlusion masks, obtained with a pretrained optical flow estimator and our heuristic.
    \emph{Left:} reference images; \emph{middle:} generated novel views; \emph{right:} magenta masks indicate the parts of novel views that were occluded from the reference point of view.
    The area of such regions in SWORD is much greater than for RealEstate10k, justifying its usage.
    }
    \label{fig:raft_mask}
\end{figure}

{\small
\bibliographystyle{ieee_fullname}
\bibliography{refs}
}
\end{document}